\documentclass[10pt, english, a4paper]{article}
\usepackage[T1]{url}
\usepackage[utf8]{inputenc}
\usepackage[hidelinks]{hyperref} 
\usepackage{graphicx}
\usepackage{varioref, babel, fancyvrb, listings, amsmath,amsthm, amssymb, enumerate, mathtools}
\usepackage{bbm}
\usepackage{hyperref}
\usepackage[left=1.25in,right=1.25in,top=1.25in,bottom=1.25in]{geometry} 
\usepackage{mathabx}
\usepackage{listings}
\usepackage{todonotes}
\usepackage{csquotes}
\usepackage{tikz}
\usepackage{array}
\usepackage{paralist}
\usepackage{enumitem}
\usepackage[normalem]{ulem}
\usepackage{scrextend}
\usepackage[labelformat=simple, labelsep = colon, font=normalsize]{subcaption}
\DeclareCaptionSubType*{figure}
\usepackage{algorithm}
\usepackage[noend]{algpseudocode}
\usepackage{scrextend}

\usetikzlibrary{patterns}

\theoremstyle{plain}
\newtheorem{theorem}{Theorem}[section]

\newtheorem{proposition}[theorem]{Proposition}

\newtheorem{conjecture}[theorem]{Conjecture}
\numberwithin{theorem}{section}

\theoremstyle{definition}
\newtheorem{definition}[theorem]{Definition}
\newtheorem{example}[theorem]{Example}

\theoremstyle{remark}
\newtheorem{remark}[theorem]{Remark}

\DefineNamedColor{named}{Purple}{cmyk}{0.45,0.86,0,0}

\DefineNamedColor{named}{JungleGreen} {cmyk}{0.99,0,0.52,0}

\DefineNamedColor{named}{orange} {rgb}{1,0.55,0}

\newcommand{\R}{\mathbb{R}}

\title{\textbf{What do AI algorithms actually learn? -- On false structures in deep learning}}

\newcommand*\samethanks[1][\value{footnote}]{\footnotemark[#1]}
\author{Laura Thesing\thanks{Department of Applied Mathematics and Theoretical Physics,  
  University of Cambridge, UK
} 
    \and Vegard Antun\thanks{Department of Mathematics, University of Oslo, Norway}
    \and Anders C. Hansen\samethanks[1]$~~$\samethanks[2] 
}

\begin{document}

\maketitle

\begin{abstract}
There are two big unsolved mathematical questions in artificial intelligence (AI): (1) Why is deep learning so successful in classification problems and (2) why are neural nets based on deep learning at the same time universally unstable, where the instabilities make the networks vulnerable to adversarial attacks. We present a solution to these questions that can be summed up in two words; false structures. Indeed, deep learning does not learn the original structures that humans use when recognising images (cats have whiskers, paws, fur, pointy ears, etc), but rather different false structures that correlate with the original structure and hence yield the success. However, the false structure, unlike the original structure, is unstable. The false structure is simpler than the original structure, hence easier to learn with less data and the numerical algorithm used in the training will more easily converge to the neural network that captures the false structure.  
We formally define the concept of false structures and formulate the solution as a conjecture. 
Given that trained neural networks always are computed with approximations, this conjecture can only be established through a combination of theoretical and computational results similar to how one establishes a postulate in theoretical physics (e.g. the speed of light is constant). Establishing the conjecture fully will require a vast research program characterising the false structures. We provide the foundations for such a program establishing the existence of the false structures in practice. Finally, we discuss the far reaching consequences the existence of the false structures has on state-of-the-art AI and Smale's 18th problem.
 \end{abstract}

\section{Introduction}

It is now well established through the vast literature on adversarial attacks 
\cite{Eykholt2018RobustPA,
Fawzi2017TheRO, 
MIT_adversary,
moosavi17, 
moosavi16, 
Nguyen2015DeepNN, 
Song2018PhysicalAE,
Su2019OnePA, 
szegedy13, 
Vargas2019UnderstandingTO}
(we can only cite a small subset here) on neural networks for image classification that deep learning provides highly successful, yet incredibly unstable neural networks for classification problems. Moreover, recently, the instability phenomenon has also been shown \cite{InvFool} for deep learning in image reconstruction and inverse problems 
\cite{Adler17,
jin17,
Lucas18,
McCann17,
Nature_highlights,
sun16,
Rosen_Nature}.
Thus, this phenomenon of instability seems to be universal. What is fascinating
is that despite intense research trying to solve the instability issue 
\cite{cubuk2018autoaugment,
goodfellow14, 
lu2017safetynet,
ma2018characterizing,
papernot2016distillation,
wang2018a},
the problem is still open. As a result, there is a growing concern regarding the consequences of the instabilities of deep learning methods in the sciences. Indeed, {\it Science} \cite{Finlayson1287} recently reported on researchers warning about potential fatal consequences in medicine due to the instabilities. Hence, we are left with the following fundamental question:
\begin{displayquote}
\normalsize
\vspace{-2mm}
{\it Why are current deep learning methods so successful in image classification, yet universally unstable and, hence, vulnerable to adversarial attacks?}
\vspace{-2mm}
\end{displayquote}
In this paper we provide a radical conjecture answering the above question in classification and explaining why this problem will not be solved unless there is a fundamental rethink of how to approach learning. We provide the first steps towards such a theory. 

\begin{conjecture}[False structures in classification]\label{con1:false}
The current training process in deep learning for classification forces the neural network to learn a different (false) structure and not the actual structure of the classification problem. There are three main components:
\begin{compactitem}
\itemsep0pt
\item[(i)]({\bf Success}) The false structure correlates well with the original structure, hence one gets a high success rate.
\item[(ii)] ({\bf Instability}) The false structure is unstable, and thus the network is susceptible to adversarial attacks. 
\item[(iii)] ({\bf Simplicity}) The false structure is simpler than the desired structure, and hence easier to learn e.g. fewer data is needed and the numerical algorithm used in the training easily converges to the neural network that captures the false structure.  
\end{compactitem}
\end{conjecture}

\begin{remark}[Structure] One can think of the word structure to mean the concept of what would describe a classification problem. Considering an image, we could think of what makes humans recognise a cat or a fire truck. In particular, a structure describing a cat would encompass all the features that make humans recognise cats. However, classification problems extend beyond image recognition. For example, one may want to classify sound patterns or patters in meteorological data, seismic data etc. Thus, we need a proper mathematical definition of what we mean by structure and also false structure.
\end{remark}

\begin{figure}
    \centering
    \includegraphics[width=0.8\textwidth]{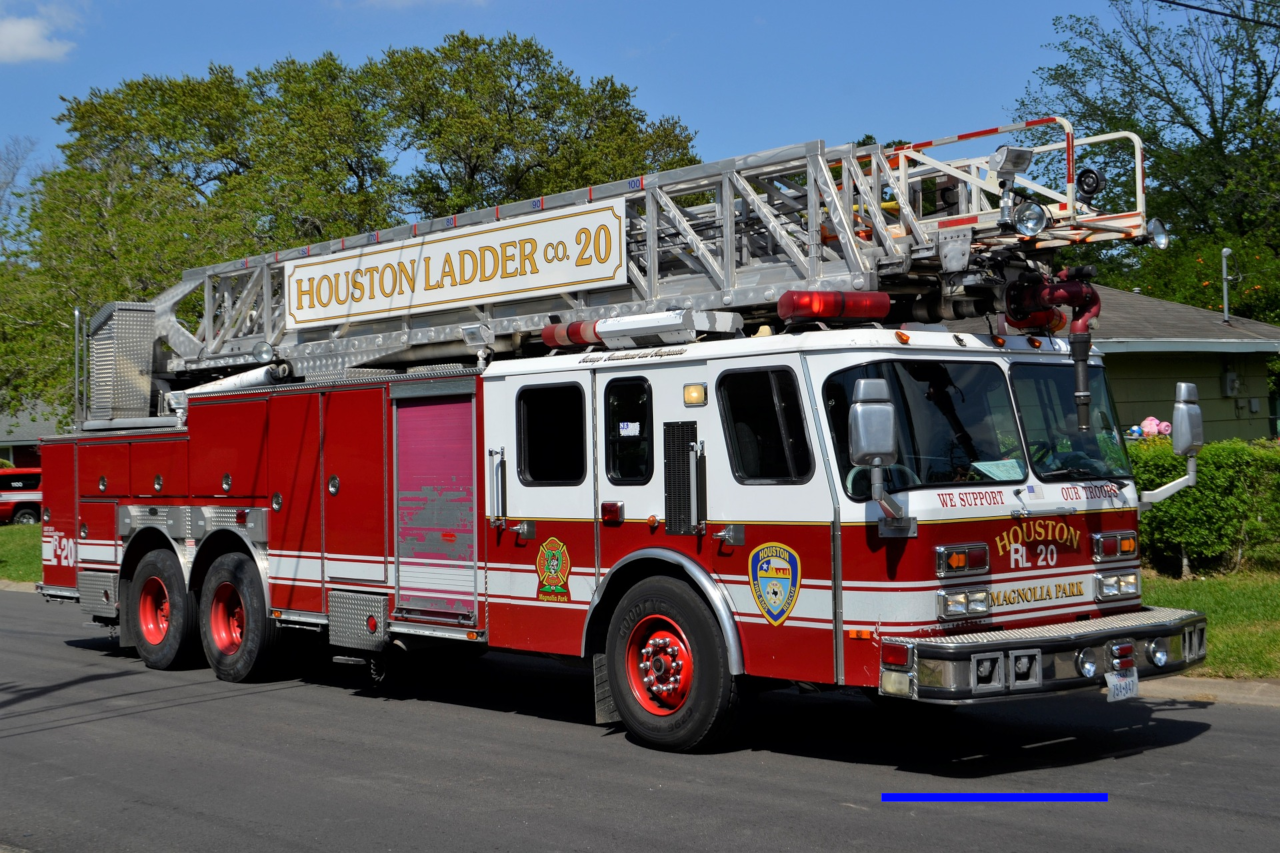} 
    \caption{Image of a fire truck with a horizontal blue line in the lower right corner, as suggested in the thought experiment in Example \ref{ex:thought}. As discussed in the example, the blue line yields many potential false structures.}
    \label{fig:firetruck}
\end{figure}
\subsection{Smale's 18th problem}
Based on a request from V. Arnold, inspired by Hilbert's list of mathematical problems for the 20th century, S. Smale created a list of mathematical problems for the 21st century \cite{21century_Smale}. The last problem on this list is more relevant than ever and echoes Turing's paper from 1950 \cite{Turing_1950} on the question of artificial intelligence. Turing asks if a computer can think, and suggests the imitation game as a test for his question about artificial intelligence. Smale takes the question even further and asks in his 18th problem the following:
\begin{displayquote}
\normalsize
{\it ``
What are the limits of intelligence, both artificial and human?"}
\\[5pt]
\rightline{ --- Smale's 18th problem (from mathematical problems for the 21st century \cite{21century_Smale}). \hspace{15mm}}
\end{displayquote}
Smale formulates the question in connection to foundations of computational mathematics and discusses different models of computation appropriate for the problem \cite{BCSS, Turing_Machine}. The results in this paper should be viewed in connection with Smale's 18th problem and the foundations of computation (see \S \ref{sec:conse}). Our contribution is part of a larger program on foundations of computational mathematics and the Solvability Complexity Index hierarchy 
\cite{SCI_optimization, 
CRAS, 
Hansen_JFA,
Hansen_JAMS, 
Odi, Rogers}  
established to determine the boundaries of computational mathematics and scientific computing.

\section{Why the concept of false structures is needed}

To illustrate the concept of false structure and the conjecture we continue with a thought experiment. 
\begin{example}[A thought experiment explaining false structures and Conjecture \ref{con1:false}]\label{ex:thought}
Suppose a human is put in a room in a foreign country with an unknown language. The person is to be trained to recognise the label "fire truck", and in order to do so, she is given a large complex collection of images with different types of fire trucks and images without fire trucks. Each time an image with a fire truck is shown to the person, she hears the foreign language word for fire truck. In particular, the person is trying to learn the function $f : \mathcal{M} \rightarrow \{0,1\}$, where $1$ means that there is a fire truck in the image, and $\mathcal{M}$ denotes a large set of images. We will refer to $f$ as the original structure. However, on the set $\mathcal{T} \subset \mathcal{M}$ of images shown in the training process there is a small, but clearly visible, horizontal blue line on each of the images showing a fire truck (this is visualised in Figure \ref{fig:firetruck}). On the images without fire trucks there is no line.  
The question is: will the person think that the foreign language word for fire truck means horizontal blue line, blue line, line or actual fire truck. This is an example of the original structure $f$ (describing fire trucks) and three false structures $g_1$ (horizontal blue line), $g_2$ (blue line, not necessary horizontal), and $g_3$ (line, colour and geometry are irrelevant). All of them could have been learned from the same data. However, the structure or false structure that the person has learned will yield wildly different results
on different tests:

\begin{labeling}{\,\,}
\item[\,\, \,(i)  ({\bf False structure $g_1$}):] Suppose the person learns the false structure $g_1$ describing the horizontal blue-line structure that is chosen. Suppose also that the test set of images is chosen such that every image containing a fire truck also has a horizontal blue line. On this test set one will have $100\%$ success, yet the false structure $g_1$ will give incredible instabilities in several different ways. First, a tiny perturbation in terms of removing the blue line will yield a miss classification. Second, a slight rotation of the blue line would mean wrong output, and a slight change in colour will result in an incorrect decision. Thus, there are at least three types of adversarial attacks that would succeed in reducing the success rate from $100\%$ to $0\%$ on the test set. 
\item[\,\, \,(ii)  ({\bf False structure $g_3$}):] Suppose it is just the line structure described by the false structure $g_3$ that is learned. Repeating the experiment in (i) with  
all images in the test set containing a fire truck also having a small visible green line would yield the same success as in (i) as well as instabilities, however, now rotations of the line would not have an effect nor changes in the colour. Thus, two less adversarial attacks would be successful. If we did the experiment with the false structure $g_2$, there would have been at least two forms of adversarial attacks available. 
\item[\,\, \,(iii)  ({\bf The original structure $f$}):] If the actual fire truck structure, described by $f$, was chosen, one would be as successful and stable as would be expected from a human when given any test set. Note that the original structure $f$ is much more complex and likely harder to learn than the false structures $g_1, g_2, g_3$. 
\end{labeling}
\end{example}

Motivated by the above thought experiment we can now formally define the original structure and false structures. 

\begin{definition}[The original structure and false structures]\label{def:false_structure}
Consider $\mathcal{M} \subset \mathbb{R}^d$ and a string of unique (see Remark \ref{rem:unique}) predicates $L = \{\alpha_0,\hdots,\alpha_N\}$ on $\mathcal{M}$, where $N \in \mathbb{N}$ such that for each $x \in \mathcal{M}$ there is a unique $\alpha_j \in L$ such that $\alpha_j(x) = \text{true}$. For such $x$ define $f(x) = j$. We say that the pair $(f,L)$ is \emph{the original structure} on $\mathcal{M}$. 
A \emph{false structure for $(f,L)$ relative to $\mathcal{T} \subsetneq \mathcal{M}$} is a pair $(g,L^{\prime})$, where $L^{\prime} = \{\beta_0,\hdots,\beta_N\}$ a string of unique predicates   with $g : \mathcal{M} \rightarrow \{0,\hdots, N\}$ such that $g(x) = j$ iff $\beta_j(x) = \text{true}$. Moreover, $\beta_j \neq \alpha_j$ for all $j \in \{0,\hdots,N\}$ and 
\begin{equation}\label{eq:iff}
\begin{split}
g(x) = j \text { iff } f(x) = j \, \forall \, x \in \mathcal{T},\\
\end{split}
\end{equation} 
as well as
\begin{equation}\label{eq:xnotin}
\mathcal{C} = \{x \in \mathcal{M}\setminus \mathcal{T} \, \vert \, f(x) = i, \, g(x) = j, \, \text{for some } i \neq j\} \neq \emptyset.
\end{equation}
We say that $g$ is a partial false structure if $\alpha_j \neq \beta_j$ for at least two different $j \in \{0,\hdots, N\}$ (as opposed to all).
\end{definition}

\begin{remark}[Unique predicates]\label{rem:unique}
By unique predicates $L = \{\alpha_0,\hdots,\alpha_N\}$ we mean that the support of the characteristic functions induced by the predicates in $L$ do not intersect. 
\end{remark}

\begin{remark}[How bad is the false structure?]
Note that Definition \ref{def:false_structure} does not consider how 'far' the false structure $g$ is from the original structure $f$. This is beyond the scope of this paper, however, this can easily be done. For example, suppose $\mathcal{M}$ is equipped with some probability measure $P$, then assuming $\mathcal{C}$ from \eqref{eq:xnotin} is measurable, $P(\mathcal{C})$ would indicate how severe it would be to learn the false structure instead of the original structure. 
\end{remark}

The motivation behind the idea of a false structure can be understood as follows. Suppose one is interested in learning the original structure $(f,L)$ as in Definition \ref{def:false_structure}. In Example \ref{ex:thought} the list $L = \{\alpha_0, \alpha_1\}$ of predicates are $\alpha_1(x) = x \text{ demonstrates a fire truck}$, and $\alpha_0(x) = x \text{ does not demonstrate a fire truck}.$ 
In order to learn $(f,L)$ we have a training set $\mathcal{T} \subset \mathcal{M}$. However, if we have a false structure relative to $\mathcal{T}$, as in Example \ref{ex:thought}, $(g_1,L^{\prime})$ where $L^{\prime} = \{\beta_0, \beta_1\}$ with 
\begin{equation*}
\begin{split}
\beta_1(x) &= x \text{ demonstrates a horizontal blue line}, \\
\quad \beta_0(x) &= x \text{ does not demonstrate a horizontal blue line},
\end{split}
\end{equation*}
 how do we know that we have not learned $(g_1, L^{\prime})$ instead? In Example \ref{ex:thought} there are three different false structures, each with its different instability issues. 

\begin{remark}[Formulation of the predicate]
By "$x$ demonstrates a $z$" (where $z$ was a horizontal blue line above) in the previous predicate we mean that the main object showing in $x$ is $z$, and that there is only one main object. The word demonstrate is slightly ambiguous, however, for simplicity we use this formulation. 
\end{remark}

Example \ref{ex:thought} illustrates the issues in Conjecture \ref{con1:false} very simply. Indeed, a simple false structure could give great performance yet incredible instabilities. Let us continue with the thought experiment, however, now we will replace the human in Example \ref{ex:thought} with a machine, and in particular, we consider the deep learning technique. 
\begin{example}
We consider the same problem as in Example \ref{ex:thought}, however, we replace the human by a neural network that we shall train. Indeed, we let  
$
 f: \mathcal{M} \subset \mathbb{R}^d \rightarrow \{0,1\}
$
be the function deciding if there is a fire truck in the image, where $\mathcal{M}$ is as in Example \ref{ex:thought}. The training set $\mathcal{T} = \{x^1, \hdots, x^r\}$ and test set $\mathcal{C} = \{y^{1}, \ldots, y^{s}\}$ consists of images $x^j$ and $y^{j}$ with and without fire trucks. However, all fire truck images also contain a small blue horizontal line, and there is no blue line in the images without fire trucks. 
We choose a cost function $C$, a class of neural networks $\mathcal{NN}$ and approximate the optimisation problem 
\begin{equation}\label{eq:minimisation_problem}
\Psi \in  \mathop{\mathrm{arg min}}_{\Phi \in \mathcal{NN}}  C(v,w)  \text{ where } v_j = \Phi(x^j), w_j = f(x^j) \text{ for } 1 \leq j \leq r.
\end{equation}
The question is now, given that there are three false structures $(g_1, L_1^{\prime})$, $(g_2,L_2^{\prime})$, $(g_3, L_3^{\prime})$ (the predicates in $L_j^{\prime}$ would come from the description in Example \ref{ex:thought}) that would also fit \eqref{eq:minimisation_problem}:
\begin{displayquote}
\normalsize
{\it Why should we think that the trained neural network has picked up the correct structure, and not any of the the false structures? 
}
\end{displayquote}
\end{example}

\begin{remark}[Simplicity]
Note that there could be a minimiser $\Psi$ of \eqref{eq:minimisation_problem} such that $\Psi = f$, however, this minimiser may be hard to reach and hence one finds another minimiser $\tilde \Psi$ of  \eqref{eq:minimisation_problem} such that $\tilde \Psi = g_1$, say. We will see examples of this phenomenon below. 
\end{remark}

\subsection{Support for the conjecture and how to establish it}
Unlike common conjectures in mathematics, Conjecture \ref{con1:false} can never be proven with standard mathematical tools. The issue is that all neural networks that are created are done so with inaccurate computations. Thus, actual minimisers are rarely found, if ever, but rather approximations in one form or another. Thus Conjecture \ref{con1:false} should be treated more like a postulate in theoretical physics, like 'the speed of light is constant'. One can never establish this with a mathematical proof, however, mathematical theory and experiments can help support the postulate. 

Note that there is already an overwhelming amount of numerical evidence that Conjecture \ref{con1:false} is true based on the myriad of experiments done over the last five years. 
Indeed, we have the following documented cases: \textit{(I) Unrecognizable and bizarre structures are labeled as natural images} \cite{Nguyen2015DeepNN}. Trained successful neural nets classify unrecognizable and bizarre structures as natural images with standard labels with high confidence. 
Such mistakes would not be possible if the neural network actually captured the correct structure that allows for image recognition in the human brain. 
\textit{(II) Perturbing one pixel changes the label} \cite{Su2019OnePA, Vargas2019UnderstandingTO}. It has been verified that trained and successful networks change the label of the classification even when only one pixel is perturbed. Clearly, the structures in an image that allows for recognition by humans are not affected by a change in a single pixel. 
\textit{(III) Universal invisible perturbations change more than $90\%$ of the labels}  \cite{moosavi17, moosavi16}. The DeepFool software \cite{Fawzi2017TheRO} demonstrates how a single almost invisible perturbation to the whole test set dramatically changes the failure rate. 
Different structures in images, allowing for successful human recognition, are clearly not susceptible to misclassification by a single near-invisible perturbation. 
However, the false structure learned through training is clearly unstable. 

There is quite a bit of work on establishing which part of the data is crucial for the decision of the classifier
\cite{ribeiro2016should,
selvaraju2017grad,
simonyan2013deep,
zeiler2014visualizing,
zhou2016learning}. 
This is a rather different program compared to establishing our conjecture. Indeed, our conjecture is about the unstable false structures. However, one should not rule out that there might be connections that could help detecting and understanding the false structures. 

\subsection{Consequences of Conjecture \ref{con1:false}}\label{sec:conse}

The correctness of Conjecture \ref{con1:false} may have several consequences both negative and positive.

\noindent {\bf Negative consequences:}
\begin{labeling}{\,\,}
\vspace{-1mm}
\item[\,\, \,(i)] The success of deep learning in classification is not due to networks learning the structures that humans associate with image recognition, but rather that the network picks up unstable false structures in images that are potentially impossible for humans to detect. This means that instability, and hence vulnerability to adversarial attacks, can never be removed until one guarantees that no false structure is learned. This means a potential complete overhaul of modern AI.  
\vspace{-1mm}
\item[\,\, \,(ii)] 
The success is dependent of the simple yet unstable structures, thus the AI does not capture the intelligence of a human.  
\vspace{-1mm}
\item[\,\, \,(iii)] Since one does not know which structure the network picks up, it becomes hard to conclude what the neural network actually learns, and thus harder to trust its prediction. What if the false structure gives wrong predictions?
\end{labeling}
\vspace{-1mm}
\noindent {\bf Positive consequences:}
\begin{labeling}{\,\,}
\vspace{-1mm}
\item[\,\, \,(i)] Deep learning captures structures that humans cannot detect, and these structures require very little data and computing power in comparison to the true original structures, however, they generalise rather well compared to the original structure. Thus, from an efficiency point of view, the human brain may be a complete overkill for certain classification problems, and deep learning finds a mysterious effective way of classifying.
\vspace{-1mm}
\item[\,\, \,(ii)] 
 The structure learned by deep learning may have information that the human may not capture. This structure could be useful if characterised properly. For example, what if there is structural information in the data that allows for accurate prediction that the original structure could not do?
 \end{labeling}
\vspace{-1mm}
\noindent {\bf Consequences - Smale's 18th problem:}
\begin{labeling}{\,\,}
\vspace{-1mm}
\item[\,\, \,(i)] Conjecture \ref{con1:false} suggests that there is a fundamental difference between state-of-the-art AI and human intelligence as neural networks based on deep learning learn completely different structures compared to what humans learn. Hence, in view of Smale's 18th problem, correctness of Conjecture \ref{con1:false} implies both limitations of AI as well as human intelligence. Indeed, the false unstable structures learned by modern AI limits its abilities to match human intelligence regarding stability. However, in view of the positive consequences mentioned above, correctness of Conjecture \ref{con1:false} implies that there is a limitation to human intelligence when it comes to detecting other structures that may provide different information than the structure detected by humans.  
\end{labeling}

\section{Establishing Conjecture \ref{con1:false} - Do false structures exist in practice?}
Our starting point for establishing Conjecture \ref{con1:false} is Theorem \ref{thrm:main} below, for which the proof captures all the three components of the conjecture. We will demonstrate how this happens in actual computations. To introduce some notation, we let $\mathcal{NN}_{\mathbf{N},L}$, with $\mathbf{N} = (N_L,N_{L-1},\dotsc,N_1,N_0)$ denote the set of all $L$-layer neural networks. 
That is, all mappings $\Phi: \mathbb{C}^{N_0} \rightarrow \mathbb{C}^{N_L}$of the form 
\[
\Phi (x) = W_L (  \rho(  W_{L-1} (\rho(     \ldots \rho(W_1(x))    ))        )  ),
\]
with
$
x \in \mathbb{R}^{N_0},
$
where the $W_j$s are affine maps with dimensions given by $\mathbf{N}$, and $\rho$ is a fixed non-linear function acting component-wise on a vector. 
We consider a binary classifier
$
 f: \mathcal{M} \subset \mathbb{R}^d \rightarrow \{0,1\},
$
where $\mathcal{M}$ is some subset. To make sure that we consider stable problems we define the family of well separated and stable sets $\mathcal{S}^f_{\delta}$ with separation at least $\delta > 0$:
\begin{equation}\label{eq:the_S}
\mathcal{S}^f_{\delta} = \{\{x^1, \hdots, x^r\} \, \vert \, \|x^j\| \leq 1, \min_{i\neq j} \|x^i - x^j\|_{\infty} \geq \delta, f(x^j+y) = f(x^j) \text{ for }\|y\|_{\infty} < \delta\}.
\end{equation}
Moreover, the cost function $C$ used in the training is in 
\begin{equation}\label{eq:CF}
\mathcal{C}\mathcal{F} := \{C: \mathbb{R}^{rN_L} \times \mathbb{R}^{rN_L} \rightarrow \mathbb{R}_+ :  C(v,w) = 0 \text{ iff } v = w\}.
\end{equation}

\begin{theorem}[Bastounis, Hansen, Vlacic \cite{SCI_optimization}]\label{thrm:main}
There is an uncountable family $\mathcal{F}$ of classification functions $f: \mathbb{R}^{N_0} \rightarrow \{0,1\}$ such that for each $f \in \mathcal{F}$ and neural network dimensions $\mathbf{N} = (N_L,N_{L-1},\dotsc,N_0)$ with $N_0,L \geq 2$, any $\epsilon > 0$, and any integers $s,r$ with $r \geq 3(N_1+1) \dotsb (N_{L-1}+1)$, there exist uncountably many non-intersecting training sets  $\mathcal{T} = \{x^1, \hdots, x^r\} \in \mathcal{S}^f_{\varepsilon(r+s)}$ of size $r$ (where $\varepsilon(n) := [(4n+3)(2n+2)]^{-1}$) and uncountably many non-intersecting classification sets $\mathcal{C} = \{y^1, \hdots, y^s\} \in \mathcal{S}^f_{\varepsilon(r+s)}$ of size $s$ such that we have the following.  For every $C \in \mathcal{CF}$ there is a neural net 
\begin{equation}\label{eq:minimisation_problem2}
\Psi \in  \mathop{\mathrm{arg min}}_{\Phi \in \mathcal{NN}_{\mathbf{N},L}}  C(v,w)  \text{ where } v_j = \Phi(x^j), w_j = f(x^j) \text{ for } 1 \leq j \leq r
\end{equation}
	such that
	$\Psi(x) = f(x) \quad \forall x \in \mathcal{T} \cup \mathcal{C}.$
	However, there exist uncountably many $v \in \mathbb{R}^{N_0}$ such that 
	\begin{equation}\label{eq:NNCVIncorrectClassification}
	|\Psi(v) -  f(v)| \geq 1/2, \qquad \|v - x\|_{\infty} \leq \epsilon \text{ for some } x \in \mathcal{T}.
	\end{equation}
	Moreover, there exists a stable neural network 
	\begin{equation}\label{eq:stable_network}
	   \hat{\Psi} \notin \mathcal{NN}_{\mathbf{N},L}\, \text{ with } \, \hat{\Psi}(x) = f(x), \,\, \forall \, x \in B_{\varepsilon(r+s)}^{\infty}(\mathcal{T} \cup \mathcal{C}),
	\end{equation}
	where $B_{\varepsilon(r+s)}^{\infty}(\mathcal{T} \cup \mathcal{C})$ denotes the $\varepsilon(r+s)$ neighbourhood, in the $l^{\infty}$ norm, of $\mathcal{T} \cup \mathcal{C}$.
\end{theorem}

The message of Theorem \ref{thrm:main} and its proof \cite{SCI_optimization} (which serves as a basis for \S \ref{sec:case1}) is summarised below. 
\vspace{-1mm}
\begin{labeling}{\,\,}
\vspace{-1mm}
\item[\,\, \,(i) ({\bf Success})] The successful neural network learns a false structures that correlates well with the true structure of the problem and hence the great success (100\% success on an arbitrarily large test set).
\vspace{-1mm}
\item[\,\, \,(ii) ({\bf Instability})] The false structure is completely unstable despite the original problem being stable (see $\mathcal{S}^f_{\delta}$ in \eqref{eq:the_S}).  Because of the training process of seeking a minimiser of the optimisation problem, the neural network learns the false structure, and hence becomes completely unstable.  Indeed, it will fail on uncountably many instances that are $\epsilon$-away from the training set. Moreover, $\epsilon$ can be made arbitrarily small. 
\vspace{-1mm}
\item[\,\, \,(ii) ({\bf Simplicity})] The false structure is very simple and easy to learn. Moreover, paradoxically, there exists another neural network with different dimensions that becomes stable and has the same success rate as the unstable network, however, there is no known way to construct this network.    
\vspace{-2mm}
\end{labeling}
 Theorem \ref{thrm:main} and its proof provides the starting point in the program on establishing Conjecture \ref{con1:false}. However, the missing part is to show that: the false structure is learned in practice, it is much easier to learn than the original structure, and that the original structure is very unlikely to be learned in the presence of the false structure in the training set. This is done in the next section.  
 
 \subsection{Establishing the conjecture: Case 1}\label{sec:case1}
We will first start by demonstrating that Conjecture \ref{con1:false} is true in the many unique classification cases suggested by Theorem \ref{thrm:main}. 
A simple example is the infinite class of functions
\begin{equation}\label{eq:the_f}
f_a : \mathcal{M} \rightarrow \left\{0,1\right\}, \quad \mathcal{M} = [b,1] \times [0,1], \quad  f_a(x) = \left\lceil \frac {a} {x_1} \right\rceil \mod 2,
\end{equation} 
for $a > 0$ and $0 < b <1$. Consider the predicates
\[
\alpha_0(x) = \left\lceil \frac {a} {x_1} \right\rceil \text{is odd},
\qquad
\alpha_1(x) = \left\lceil \frac {a} {x_1} \right\rceil \text{is even}.
\]
Let $L = \{\alpha_1, \alpha_2\}$ then $(f_a,L)$ is the original structure, as in Definition \ref{def:false_structure}.
We note that $f_a$ is constant on
each of the intervals 
$
\frac{a}{k+1} \leq x_1 < \frac{a}{k}, 
$
where $k \in \mathbb{N},$
and hence may be viewed as a very simple classification problem: given $x$, is $f_a(x) = 1$ or $f_a(x) = 0$? 
To simplify the learning and analysis further we will assume that $a \in \mathbb{N}$, $a < K$ for some chosen $K \in \mathbb{N}$, and 
we let $b=\tfrac{a}{K+1}$. This means that $f_a$ has $K-a$ jump 
discontinuities on the interval $[b,1)$. To ensure $f_a$
is stable with respect to perturbations of size $\epsilon > 0$ on its input, 
we will ensure that each of our samples of $f_a$ lies at least $\epsilon$ away from
each of these jump discontinuities. Hence, we define \begin{equation}\label{eq:stable_area}
\mathcal{S_{\epsilon}} = 
\bigcup_{k=a}^{K} \left(\frac{a}{k+1} + \epsilon, \frac{a}{k} - \epsilon\right). 
\end{equation}
choose the samples from the set $\mathcal{S}_{\epsilon} \times [0,1]$.
This is similar to \eqref{eq:the_S} used in Theorem \ref{thrm:main}. 
To avoid that $\mathcal{S}_{\epsilon}$ is a union of empty sets we will 
always assume that $\epsilon < b^2/(2(a-b))$.
It will not be a goal in itself to learn $f_a$
for any $x$, but rather to learn the right value of $f_a$ within the
$\epsilon$-stable region $\mathcal{S_{\epsilon}}$. Indeed, given that this is a decision problem, inputs close to the boundary are always going to be hard to classify. However, the decision problem stays stable on $\mathcal{S_{\epsilon}}\times [0,1]$.

\begin{figure}[tbp]
    \begin{tabular}{m{0.01\textwidth}m{0.20\textwidth}m{0.20\textwidth}
                                     m{0.20\textwidth}m{0.20\textwidth}}
&$\quad\quad\quad\quad\quad\mathcal{T}_{\delta}^{7}$ &
$\quad\quad\quad\quad\quad\mathcal{T}_{0}^{5000}$ &
$\quad\quad\quad\quad\quad\mathcal{T}_{\delta}^{5000}$ &
$\quad\quad\quad\quad\quad\mathcal{T}_{0}^{10000}$ \\
$\mathcal{C}_{0}$ &
  \includegraphics[width=0.235\textwidth]{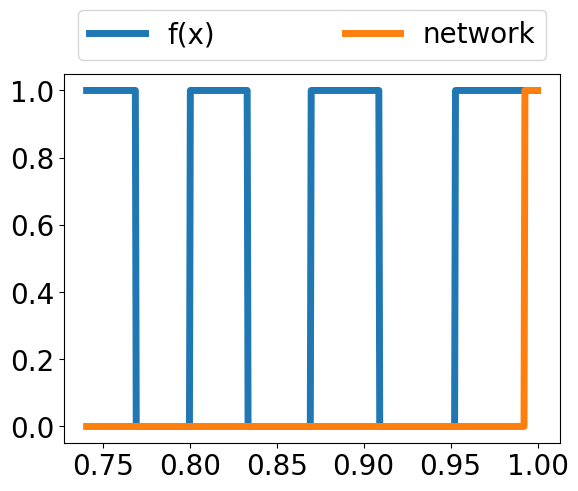}
 &\includegraphics[width=0.235\textwidth]{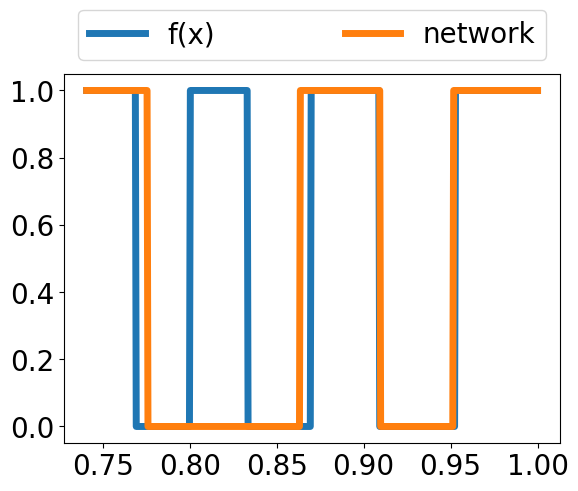}
 &\includegraphics[width=0.235\textwidth]{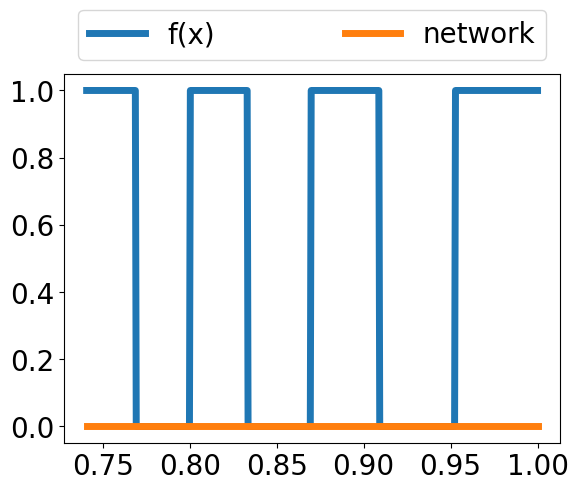}
 &\includegraphics[width=0.235\textwidth]{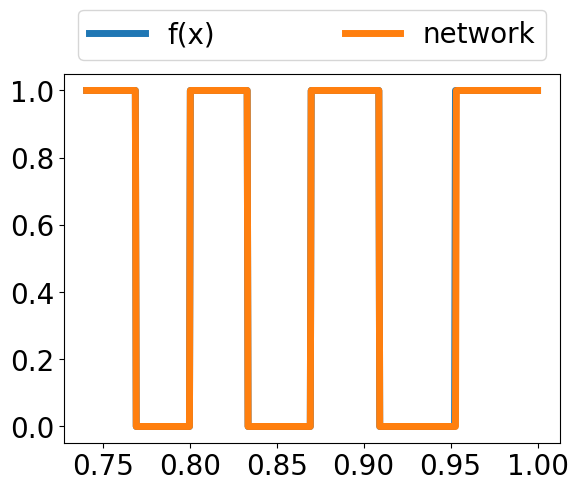} \\
$\mathcal{C}_{\delta}$ &
  \includegraphics[width=0.235\textwidth]{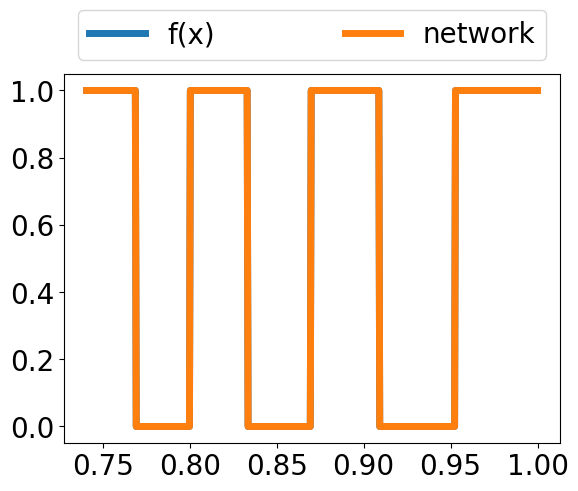} 
 &\includegraphics[width=0.235\textwidth]{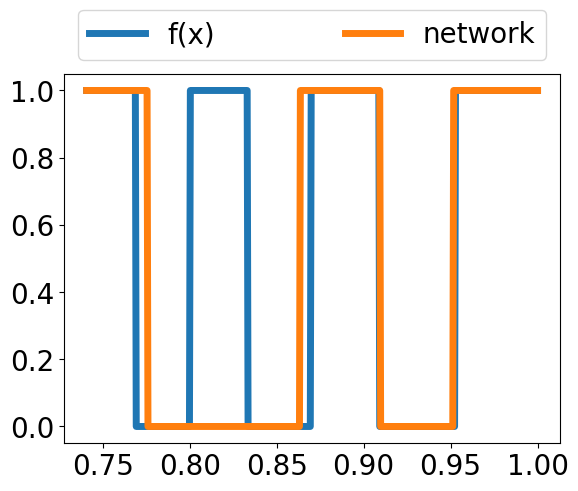} 
 &\includegraphics[width=0.235\textwidth]{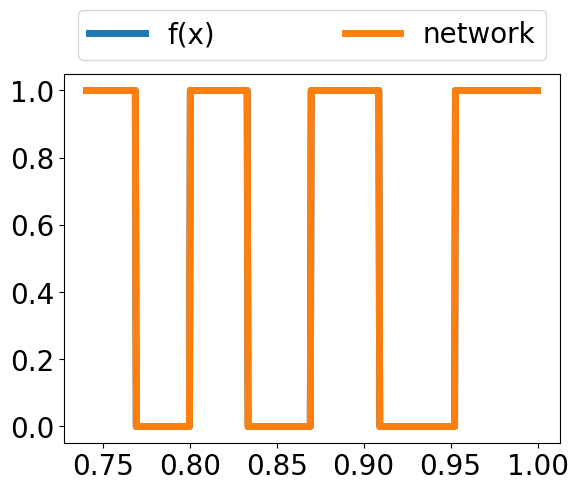} 
 &\includegraphics[width=0.235\textwidth]{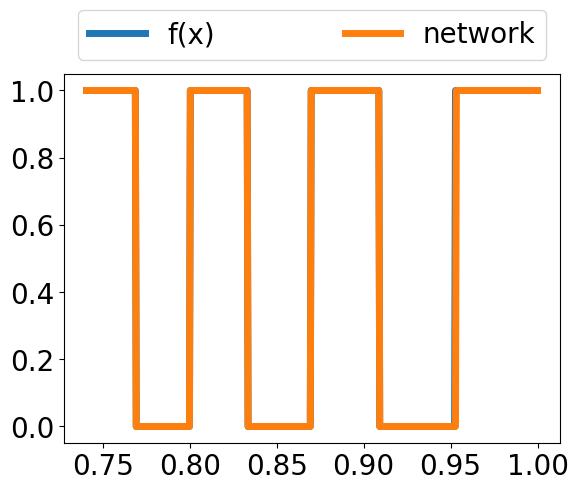} 
    \end{tabular}
\caption{\label{fig:case1}
({\bf Experiment I}) The graphs shows the output of $f_a(x)$ and $\lfloor \sigma(\Psi_i(x))\rceil$
for $x$ in the two sets $\mathcal{C}_{0}$ (top row) and $\mathcal{C}_{\delta}$
(bottom row).  The networks $\Psi_i$, $i=1,\ldots, 4$ have been trained on the
sets $\mathcal{T}_{\delta}^{7}$, $\mathcal{T}_{0}^{5000}$,
$\mathcal{T}_{\delta}^{5000}$ and $\mathcal{T}_{\delta}^{10000}$, respectively,
and are shown from left to right.  
}
\end{figure}

For the learning task we now consider two sets of size $r$, 
\begin{equation}\label{eq:trainingset}
\mathcal{T}_{0}^{r} = \{ (x_1^{i}, 0)\}_{i=1}^{r},
 \quad
\mathcal{T}_{\delta}^{r} = \{ (x_1^{i}, x_2^{i} )\}_{i=1}^{r}, 
\,\, x_2^{i} = \delta f_{a}(x_1^{i}), \qquad x_1^{i} \in \mathcal{S}_{\epsilon},
\end{equation}
where $0 < \delta < \epsilon$. Note that $\mathcal{T}_{\delta}^{r} $ gives rise to a false structure as the next proposition shows. 

\begin{proposition}\label{prop:pred1}
Consider the predicates $\beta_0$ and $\beta_1$ on $\mathcal{M}$ defined by 
\begin{align*}
\beta_0(x) = x_2 \text{ is } 0, \qquad
\beta_1(x) = x_2 \text{ is not } 0. 
\end{align*}
Define $g : \mathcal{M} \rightarrow \{0,1\}$ by $g(x) = 0$ when $\beta_0(x)=\text{true}$  and $g(x) = 1$ when $\beta_1(x)=\text{true}$. Let $L^{\prime} = \{\beta_0, \beta_1\}$. Then $(g,L^{\prime})$ is a false  
structure for $(f_a,L)$ relative to $\mathcal{T}_{\delta}^{r}$.
\end{proposition}

Note that for small $\delta$, $g$ becomes unstable on $\mathcal{T}_{\delta}^{r}$. Moreover, the false structure $(g,L^{\prime})$ appears much simpler than the original structure $f_a$. In order to train a neural network to learn the original structure $f_a$ we choose the set of networks (architecture) to be $\mathcal{NN}_{\mathbf{N},2}$. That is all fully connected 2-layer networks with dimensions $\mathbf{N} = [1, 4K, 2]$ and non-linear function $\rho: \mathbb{R} \rightarrow \mathbb{R}$ being the ReLU function. 
What is crucial is that the set $\mathcal{NN}_{\mathbf{N},2}$ is rich enough to
predict the value of $f_a$. By prediction we here mean that for a network
$\phi$ the value of $\lfloor \sigma \circ \phi \rceil$, agrees with $f_a$,
where $\sigma(x) = 1/(1+\exp(-x))$, $x \in \R$ is the sigmoid function, and
$\lfloor \cdot \rceil$ means rounding to nearest integer.  For the function
class $\mathcal{NN}_{\mathbf{N},2}$ above, it is indeed, possible to find such
a network whose prediction agrees with $f_a$ on the stable area
$\mathcal{S}_{\epsilon} \times [0,1]$ in \eqref{eq:stable_area}. This is
formalised in the following statement. 

\begin{proposition}[Existence of stable and accurate network]
\label{prop:good_network}
Let $\sigma(x) = 1/(1+\exp(-x))$, $x\in \R$ be the sigmoid function, and let
$C \colon \R^{r} \times \R^{r} \to \R$ be the cross entropy cost function for binary classification, that is
\begin{equation}\label{eq:cost_function}
C(v, w) = \sum_{j=1}^{r} -w_{j} \log(\sigma(v_j)) - (1-w_{j})\log(1 -\sigma(v_{j})).
\end{equation}
Let $f_a$ be as in \eqref{eq:the_f}.  Then, for any $\eta >
0$, there exists a two layer neural network $\Psi \in \mathcal{NN}_{\mathbf{N},
2}$ with $\mathbf{N} = [1, 4K, 2]$ using the ReLU activation function, such that 
\[
f_a(x) = \lfloor \sigma ( \Psi(x)) \rceil, \quad x \in \mathcal{S}_{\epsilon} \times [0,1], 
\]
where $\lfloor \cdot \rceil$ denotes rounding to the closest integer,  
and for any subset $\mathcal{T} = \{x^{(1)}, \ldots, x^{(r)}\}\subset
\mathcal{S}_{\epsilon} \times [0,1]$, $v = \{\Psi(x)\}_{x \in \mathcal{T}}$ and $w
= \{ f_a(x)\}_{x \in \mathcal{T}}$ we have $C(v, w) \leq \eta$.
\end{proposition}

In particular, Proposition \ref{prop:good_network} states, similarly to the last statement in Theorem \ref{thrm:main}, that there is a stable and accurate network approximating $f_a$ that provides arbitrarily small values of the cost function $C$.
 The problem, however, as we will see in the next experiment, is that it is very unlikely to be found it in the presence of the false structure. 

\begin{figure}[t]
\centering
    \begin{tabular}{m{0.21\textwidth}m{0.21\textwidth}m{0.21\textwidth}m{0.21\textwidth}}
    $\quad \Phi(x_1) = \text{vertical}$ &  $\, \Phi(x_2) = \text{horizontal}$ & $\quad \Phi(x_3) = \text{vertical}$ & $\, \Phi(x_4) = \text{horizontal}$ \\
     \includegraphics[width=0.23\textwidth]{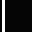} 
    &  \includegraphics[width=0.23\textwidth]{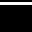}  
    &  \includegraphics[width=0.23\textwidth]{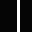}
    &  \includegraphics[width=0.23\textwidth]{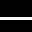}   \\
     $\,\, \Phi(x_5) = \text{horizontal}$ &  $\quad \Phi(x_6) = \text{vertical}$ & $\,\, \Phi(x_7) = \text{horizontal}$ & $\quad \Phi(x_8) = \text{vertical}$ \\
\includegraphics[width=0.23\textwidth]{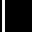} 
    &  \includegraphics[width=0.23\textwidth]{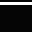}  
    &  \includegraphics[width=0.23\textwidth]{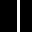}
    &\includegraphics[width=0.23\textwidth]{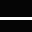}   
\end{tabular}
\caption{({\bf Experiment II}) Upper figures: Elements $\{x^j\}_{j=1}^4 \subset \mathcal{\tilde{C}}^{b,c}$ ($b=0.009$, $c = 0.01$), as well as the network classifications $\Phi(x^j)$ . Lower figures: Elements $\{x^j\}_{j=5}^8 \in \mathcal{\hat{C}}^{b,c}$ as well as the network classifications $\Phi(x^j)$. Note how the network is completely unstable as the label changes with perturbations that are not visible to the human eye.  
}
\label{fig:hor_vert}
\end{figure}

\subsubsection{Experiment I} The experiment is done as follows. We fixed $a=20$, $K=26$, 
$\epsilon = 10^{-2}$ and $\delta = 10^{-4}$ and trained four neural 
networks $\Psi_1, \Psi_2, \Psi_3, \Psi_4 \in \mathcal{NN}_{\mathbf{N},
2}$ with $\mathbf{N} = [1, 4K, 2]$ and ReLU activation function. The networks
$\Psi_i$, $i=1,\ldots, 4$ were trained on the sets $\mathcal{T}_{\delta}^{7}$,
$\mathcal{T}_{0}^{5000}$, $\mathcal{T}_{\delta}^{5000}$ and
$\mathcal{T}_{0}^{10000}$, respectively. For the set
$\mathcal{T}_{\delta}^{7}$, we ensured that the first components $x_1^{i}$ were located in separate intervals of $\mathcal{S}_{\epsilon}$.
Otherwise it would be infeasible to learn $f_a$ from $7$ samples.  For the other
sets we distributed the first components approximately equally between the
disjoint intervals of $\mathcal{S}_{\epsilon}$. To investigate which structure the
networks had learned, we define the two sets,
\begin{equation}\label{eq:testset} 
 \mathcal{C}_{0} = \{(x_1,0) \, \vert \, x_1 \in [b,1]\}\quad \text{and}\quad \mathcal{C}_{\delta} = \{(x_1,x_2) \, \vert \, x_1 \in [b,1], x_2 = \delta f_{a}(x_1) \}.
\end{equation}
We have plotted $f_a(x)$ and $\lfloor \sigma(\psi_{i}(x))\rceil$  for
$x$ in $\mathcal{C}_{0}$ and $\mathcal{C}_{\delta}$. The results are displayed in Figure \ref{fig:case1} and give the following conclusions. 
If a network 
has learned the false structure $(g,L^{\prime})$, then 
$\lfloor \sigma(\Psi_{i}(x))\rceil$ should agree with $f_{a}$ on 
$\mathcal{C}_{\delta}$, while it should be all zero on $\mathcal{C}_{0}$.
On the other hand if the network has not learned $(g,L^{\prime})$, then it should have the same output on both $\mathcal{C}_{\delta}$ and $\mathcal{C}_{0}$. If the network has learned $(f_a,L)$, it should agree with $f_a$ on both $\mathcal{C}_{\delta}$ and $\mathcal{C}_{0}$.

{\bf Conclusion (Exp I):}  $\Psi_1$, trained on $\mathcal{T}_{\delta}^{7}$ ($7$ samples) learns the false structure $(g,L^{\prime})$. $\Psi_2$, trained on $\mathcal{T}_{0}^{5000}$, does not learn the original structure $(f_a,L)$ nor the false structure $(g,L^{\prime})$. $\Psi_3$, trained on $\mathcal{T}_{\delta}^{5000}$, learns the false structure $(g,L^{\prime})$. $\Psi_4$, trained on $\mathcal{T}_{\delta}^{10000}$, learns the original structure $(f_a,L)$. Note that the conclusion supports Conjecture \ref{con1:false}: The false structure $(g,L^{\prime})$ for $(f_a,L)$ relative to $\mathcal{T}_{\delta}^{r}$ is unstable and simple and is learned with only $7$ samples. It also gives fantastic success on $\mathcal{C}_{\delta}$. The original structure $(f_a,L)$ is difficult to learn ($10000$ samples are needed to succeed, yet $5000$ samples are too few).  
Moreover, when training on $\mathcal{T}_{\delta}^{r}$, it is impossible to learn $(f_a,L)$, even with an excessive amount of samples, and the false structure $(g,L^{\prime})$ is always learned. In particular, despite 
Proposition \ref{prop:good_network} assuring that the architecture chosen is rich enough to include stable networks that approximate $f_a$ well, and should be found when the cost function is small, the good network is not found in nearly all cases.

\subsection{Establishing the conjecture: Case 2}
\renewcommand\thesubfigure{\figurename~\thefigure\alph{subfigure}}
  \begin{figure}
  \centering
 	\begin{subfigure}{0.45\textwidth}
    
 	\centering
    \begin{tikzpicture}[scale=0.6]
        \filldraw[black!15] (0, 2.25) rectangle (4.5, 3.375);
\filldraw[black!100] (0, 0) rectangle (4.5, 2.25);
\filldraw[black!100] (0, 3.375) rectangle (4.5, 4.5);
\draw[white] (2.8125, 0.5625) node {{\tiny$-a$}};
\draw[white] (2.8125, 1.6875) node {{\tiny$-a$}};
\draw[black] (2.8125, 2.8125) node {{\tiny$1-a$}};
\draw[white] (2.8125, 3.9375) node {{\tiny$-a$}};
\draw[gray!80] (0, 1.125) -- (4.5, 1.125);
\draw[gray!80] (1.125, 0) -- (1.125, 4.5);
\draw[gray!80] (0, 2.25) -- (4.5, 2.25);
\draw[gray!80] (2.25, 0) -- (2.25, 4.5);
\draw[gray!80] (0, 3.375) -- (4.5, 3.375);
\draw[gray!80] (3.375, 0) -- (3.375, 4.5);
\draw (0, 0) rectangle (4.5, 4.5);
\filldraw[black!0] (8.275, 0) rectangle (9.4, 4.5);
\filldraw[black!85] (4.9, 0) rectangle (8.275, 4.5);
\filldraw[black!85] (9.4, 0) rectangle (9.4, 4.5);
\draw[white] (5.4625, 2.8125) node {{\tiny$a$}};
\draw[white] (6.5875, 2.8125) node {{\tiny$a$}};
\draw[white] (7.7125, 2.8125) node {{\tiny$a$}};
\draw[black] (8.8375, 2.8125) node {{\tiny$1+a$}};
\draw[gray!80] (4.9, 1.125) -- (9.4, 1.125);
\draw[gray!80] (6.025, 0) -- (6.025, 4.5);
\draw[gray!80] (4.9, 2.25) -- (9.4, 2.25);
\draw[gray!80] (7.15, 0) -- (7.15, 4.5);
\draw[gray!80] (4.9, 3.375) -- (9.4, 3.375);
\draw[gray!80] (8.275, 0) -- (8.275, 4.5);
\draw (4.9, 0) rectangle (9.4, 4.5);
    \end{tikzpicture}
        \caption{\label{fig:colour}The colour code for the images defining $\mathcal{\tilde M}$. Horizontal stripe: the light coloured pixels have value $1-a$ and the dark coloured pixels have value $-a$. Vertical stripe: the light coloured pixels have value $1+a$ and the dark coloured pixels have value $a$.}
 	\end{subfigure}%
 	$\quad$
    \begin{subfigure}{0.45\textwidth}
 	\centering
    \begin{tikzpicture}[scale=0.6]
        \filldraw[black!0] (0, 2.25) rectangle (4.5, 3.375);
\filldraw[black!85] (0, 0) rectangle (4.5, 2.25);
\filldraw[black!85] (0, 3.375) rectangle (4.5, 4.5);
\draw[white] (2.8125, 0.5625) node {{\tiny$a$}};
\draw[white] (2.8125, 1.6875) node {{\tiny$a$}};
\draw[black] (2.8125, 2.8125) node {{\tiny$1+a$}};
\draw[white] (2.8125, 3.9375) node {{\tiny$a$}};
\draw[gray!80] (0, 1.125) -- (4.5, 1.125);
\draw[gray!80] (1.125, 0) -- (1.125, 4.5);
\draw[gray!80] (0, 2.25) -- (4.5, 2.25);
\draw[gray!80] (2.25, 0) -- (2.25, 4.5);
\draw[gray!80] (0, 3.375) -- (4.5, 3.375);
\draw[gray!80] (3.375, 0) -- (3.375, 4.5);
\draw (0, 0) rectangle (4.5, 4.5);
\filldraw[black!15] (8.275, 0) rectangle (9.4, 4.5);
\filldraw[black!100] (4.9, 0) rectangle (8.275, 4.5);
\filldraw[black!100] (9.4, 0) rectangle (9.4, 4.5);
\draw[white] (5.4625, 2.8125) node {{\tiny$-a$}};
\draw[white] (6.5875, 2.8125) node {{\tiny$-a$}};
\draw[white] (7.7125, 2.8125) node {{\tiny$-a$}};
\draw[black] (8.8375, 2.8125) node {{\tiny$1-a$}};
\draw[gray!80] (4.9, 1.125) -- (9.4, 1.125);
\draw[gray!80] (6.025, 0) -- (6.025, 4.5);
\draw[gray!80] (4.9, 2.25) -- (9.4, 2.25);
\draw[gray!80] (7.15, 0) -- (7.15, 4.5);
\draw[gray!80] (4.9, 3.375) -- (9.4, 3.375);
\draw[gray!80] (8.275, 0) -- (8.275, 4.5);
\draw (4.9, 0) rectangle (9.4, 4.5);
    \end{tikzpicture}
        \caption{\label{fig:colour_false}The colour code for the images defining $\mathcal{\widehat M}$. Horizontal stripe: the light coloured pixels have value $1+a$ and the dark coloured pixels have value $a$. Vertical stripe: the light coloured pixels have value $1-a$ and the dark coloured pixels have value $-a$.}
 	\end{subfigure}
 \end{figure}
Consider $\mathcal{M}$ to be the collection of $32\times32$ grey scale images with a $3$-pixel wide either  horizontal or vertical light stripe on a dark background as shown in Figure \ref{fig:hor_vert}. The colour code is as follows: $-0.01$ is black and $1+0.01$ is white. Hence, numbers between $-0.01$ and $0.01$ yield variations of black and numbers between $1-0.01$ and $1+0.01$ give variations of white. Thus, there are slight differences in the black and white colours, however, they are typically not visible to the human eye.

Define the original structure $(f,L)$ on $\mathcal{M}$ with $L = \{\alpha_0, \alpha_1\}$, where 
\begin{equation}\label{eq:f_orig}
\alpha_0(x) = x \text{ has a light horizontal stripe}, \quad
\alpha_1(x) = x \text{ has a light vertical stripe},
\end{equation}
and note that the original structure $(f,L)$ is very robust to any small perturbations.

\subsubsection{Experiment II} The experiment is inspired by an example from \cite{Fawzi2017TheRO}, and is done as follows. Define the two sets
\begin{align*}
\mathcal{\tilde M} &= \{x \in \mathcal{M} \, \vert \, x \text{ has a specific colour code described in \ref{fig:colour}  with } a>0\},\\
\mathcal{\widehat M} &= \{x \in \mathcal{M} \, \vert \, x \text{ has a specific colour code described in \ref{fig:colour_false}  with } a>0\},
\end{align*}
and notice that both are non-intersecting subsets of $\mathcal{M}$. 
Next let 
$\mathcal{\tilde C}^{b,c} \subset \mathcal{\tilde{M}}$ and 
$\mathcal{\widehat C}^{b,c} \subset \mathcal{\widehat{M}}$, each contain 1000 elements
from their respective sets and where the value of $a \in [b,c]$. Each element of 
$\mathcal{\tilde C}^{b,c}$ and  $\mathcal{\widehat C}^{b,c}$, is chosen with an 
equal probability of being a vertical or horizontal stripe, and the value 
of $a$ is chosen from the uniform distribution on $[b,c]$. 

We have trained a neural network $\Phi$ on a training set   $\mathcal{T}$ which
contains exactly the $60$ unique elements in $\mathcal{\tilde M}$ for which $a
= 0.01$. The network $\Phi$ has a $100\%$ success rate on the set
$\mathcal{\tilde{C}}^{b,c}$, $b=0.009$, $c=0.01$, yet its success rate on
$\mathcal{\widehat{C}}^{b,c}$ is $0\%$. As is evident from Figure
\ref{fig:hor_vert}, $\Phi$ misclassifies images that look exactly the same as
the ones that it successfully classifies. Hence, it is completely unstable and
$\Phi$ has clearly not learned the original structure $(f,L)$. Thus, a
pertinent question is: 
\begin{displayquote}
\normalsize
\vspace{-2mm}
{\it What is the false structure that $\Phi$ has learned?  
}
\vspace{-2mm}
\end{displayquote}
To answer the question we begin with the following proposition. 
\begin{proposition}
\label{prop:pred2}
Consider the predicates $\beta_0$ and $\beta_1$ on $\mathcal{M}$ defined by 
\[
\beta_0(x) = \text{The sum of the pixel values of $x$ are } \leq 96,
\] 
\[
\beta_1(x) = \text{The sum of the pixel values of $x$ are } > 96,
\] 
and let $L^{\prime} = \{\beta_0, \beta_1\}.$ Let $g(x) = 0$ when $\beta_0(x) = \text{true}$, and $g(x) = 1$ when $\beta_1(x) = \text{true}$. 
Then $(g,L^{\prime})$ is a false structure for $(f,L)$ relative to $\mathcal{\tilde M}$.
\end{proposition}

\begin{table}
\setlength\extrarowheight{3pt}

\begin{center}
\begin{tabular}{|c|c|c|c|} \hline
$b$ & $c$ & $\mathcal{\tilde{C}}^{b,c}$ & $\mathcal{\widehat{C}}^{b,c}$ \\ \hline
0.009 & 0.010 &  100\%  &  0.0\% \\ \hline 
0.008 & 0.009 &  100\%  &  0.0\% \\ \hline 
0.007 & 0.008 &  98.7\% &  0.0\% \\ \hline 
0.006 & 0.007 &  98.5\% &  0.6\% \\ \hline 
\end{tabular}

\caption{\label{tab:success}
 Accuracy of the network $\Phi$, on the two test sets $\mathcal{\tilde{C}}^{b,c}$
and $\mathcal{\widehat{C}}^{b,c}$, for various values of $b$ and $c$. In all cases the two sets
contain 1000 elements. }
\end{center}
\end{table}

Note that, contrary to the original structure $(f,L)$, that considers the
geometry of the problem, the false structure $(g,L)$ is clearly completely
unstable on $\mathcal{\tilde M}$. Indeed, a tiny perturbation in the pixel
values will change the label.  The question is whether it is this false
structure that is actually learned by the network $\Phi$. This turns out to be
a rather delicate question. Indeed, 
the fact that we get $100\%$ success rate on $\mathcal{\tilde{C}}^{b,c}$,
$b=0.009, c=0.01$ and $0\%$ success rate on $\mathcal{\widehat{C}}^{b,c}$
suggests that  the false structure that $\Phi$ learns is $(g,L^{\prime})$.
However, by choosing smaller values of $c$ and $d$, we see from Table \ref{tab:success} 
that the success rate of $\Phi$ on $\mathcal{\tilde{C}}^{b,c}$ decreases, whereas 
the success rate of 
$\mathcal{\widehat{C}}^{b,c}$ increases. This implies that this is not entirely the case; if $\Phi$
had learned the false structure $(g,L')$ it should have $100\%$ success rate 
on $\mathcal{\tilde{C}}^{b,c}$ and  0\% success rate on $\mathcal{\widehat{C}}^{b,c}$ for all
$0<b<c\leq0.01$.  This example illustrates how delicate the task of determining
exactly the false structure actually is, even on the simplest examples. The
actual false structure learned by $\Phi$ is likely not too far from
$(g,L^{\prime})$, but making this statement mathematically rigorous, as well as
a full test is beyond the scope of this paper.

{\bf Conclusion (Exp II):} The above numerical examples suggest that $\Phi$ learns a false structure, however, it may not always be $(g,L^{\prime})$ from Proposition \ref{prop:pred2}. 
It should also be noted that if the experiments are done with larger test sets $\mathcal{\tilde{C}}^{b,c}$ and  $\mathcal{\hat{C}}^{b,c}$ the conclusion stays the same. 
The experiments support Conjecture \ref{con1:false} as the false structures are simple to learn (only $60$ samples needed) and completely unstable.  Indeed, tiny perturbations make the network change its label. Moreover,  the network $\Phi$ that learned the false structures
become successful on large test sets.

 \section{Final conclusion}
The correctness of Conjecture \ref{con1:false} may have far reaching consequences on how we understand modern AI, and in particular on how to get to the heart of the problem of universal instability throughout neural networks based on deep learning.  The conjecture is inspired by Theorem \ref{thrm:main} and its proof in addition to the many numerical examples demonstrating instabilities in deep learning and suggesting learning of false structures. This paper provides the foundations for a larger program to establish the conjecture fully. However, as we have demonstrated in this paper, Conjecture \ref{con1:false} appears to be true even in the simplest cases.

\newpage

\bibliographystyle{abbrv}
\bibliography{references}

\section{Appendix}
\subsection{Proofs}
\subsubsection{Proof of Proposition \ref{prop:pred1}}
\begin{proof}[Proof of Proposition \ref{prop:pred1}]
    First, we notice that the predicates in $L'$ are unique, i.e. $x_2$ is
either $0$ or not and hence the classification is unambiguous. Next, we want to
see that $g(x)=j \iff f(x)=j$ on $\mathcal T_\delta^r$. We have that $x_2 =
\delta f_a(x_1)$ for all $x \in \mathcal T_\delta^r$. Hence for the first case
that we get $g(x)=0$ when $x_2 = 0$, which means that $f(x)=0$ since
$\delta > 0$. The other case follows from the same argument. Last, we see that
$\mathcal C$ is non-empty. Let $x_2 = 0$, and choose $x_1 \in [b,1]$ such that 
$f_{a}(x) =1$ for $x = (x_1,x_2)$. Such an $x_1$ will always exists since $K>a$. Then $f(x)=1$
and $g(x)=0$.
\end{proof}

\subsubsection{Proof of Proposition \ref{prop:good_network}}
\begin{proof}[Proof of Proposition \ref{prop:good_network}]
Let $x = (x_1,x_2)\in [b,1]\times [0,1]$, and  let
\begin{equation}
\phi_{\epsilon}^{c,d}(x) = \rho \left( \frac {x_1-c} {2\epsilon} + \frac 1 2 \right) - \rho \left( \frac {x_1-c} {2\epsilon} - \frac 1 2 \right) - \rho \left( \frac {x_1-d} {2\epsilon} + \frac 1 2 \right) + \rho \left( \frac {x_1-d} {2\epsilon} - \frac 1 2 \right), 
\end{equation}
where $\epsilon > 0$ and $\rho(t) = \max\{0, t\}$, $t \in \mathbb{R}$  is the ReLU
activation function.  We start by noticing that $\phi_{\epsilon}^{c,d}$ is a two layered
neural network lying in $\mathcal{NN}_{\mathbf{N},2}$ with
$\mathbf{N}=[1,4,2]$, where the coefficients in front of $x_2$ are all zero.
Hence the support of $\phi_{\epsilon}^{c,d}$ equals  $[c-\epsilon, d+ \epsilon] \times
[0,1]$. We also have that  $\phi_{\epsilon}^{c,d}(x) = 1$ for $x \in [c,d]\times [0,1]$ and
that the range of $\phi_{\epsilon}^{c,d}$ is $[0,1]$. Next let 
\begin{equation}
\Phi(x) = \sum_{k=a}^{K}  
\frac{1}{2}((-1)^{k} + 1) \phi_{\epsilon}^{c(k),d(k)}(x) 
\end{equation}
with 
\[c(k) = \frac{a}{k+1} + \epsilon \quad \text{ and }\quad  d(k) = \frac{a}{k} - \epsilon\]
be a sum of $K-a+1$ smaller networks with non-overlapping support in the first
variable.  We note that the coefficients in front of some of the functions will
be zero, and hence could have been removed, but we will not bother to do so. We notice that $\Phi \in \mathcal{NN}_{[1,4(K-a+1)],2}
\subset \mathcal{NN}_{\mathbf{\mathbf{N}},2}$ with $\mathbf{N} = [1,4K,
2]$, where the inclusion holds since $K > a$. 

Let $x_1 \in (a/(k+1)+\epsilon, a/k - \epsilon)$ and notice that
if $k+1$ is odd, then $\Phi (x) =  f_{a}(x)=1$, and if $k+1$ is even,
then $ \Phi (x) =  f_{a}(x)=0$. Hence we conclude that $\Phi
\in \mathcal{NN}_{\mathbf{N}, 2}$, with $\mathbf{N}=[1,4K,2]$, is a neural
network such that $\Phi(x) = f_{a}(x)$ for all $x \in
\mathcal{S}_{\epsilon} \times [0,1]$. 

Next let $N_1 > 0$ be a constant so that $-\log(\sigma(N_1)) < \eta/r$ and let 
$N_{2}< 0$ be a constant so that $-\log(1 - \sigma(N_2)) < \eta/r$. Furthermore 
let $N = \max\{N_1, -N_2\}$. Then then
\[ \Psi(x) = 2N \Phi(x) - N , \quad x \in [b,1]\times [0,1]\]
is a neural network in $\mathcal{NN}_{\mathbf{N}, 2}$, $\mathbf{N} = [1,4K,2]$.
Finally notice that for any $\mathcal{T} = \{x^{1}, \ldots, x^{r}\} \subset
\mathcal{S}_{\epsilon}\times [0,1]$ we have $w = \{f(x^{i})\}_{i=1}^{r} \subset
\{0,1\}^r$. Hence if  $w_{j} = 0$, then $\Psi(x^{i}) \leq -N$ and if $w_{j} = 1$
then $\Psi(x^{j}) \geq N$. letting $v = \{\Psi(x^{i})\}_{i=1}^{r}$, 
we readily see that 
\[ C(w,v) = \sum_{j=1}^{r} -w_{j}\log(\sigma(v_j)) - (1-w_j) \sigma(1-v_j) < \sum_{j=1}^{r} \eta/r = \eta, \]  
moreover, by the same argument as above we see that $\lfloor \sigma(\Psi(x)) \rceil = f(x)$ 
for all $x \in \mathcal{S}_{\epsilon}\times [0,1]$.
\end{proof}

\subsubsection{Proof of Proposition \ref{prop:pred2}}
\begin{proof}[Proof of Proposition \ref{prop:pred2}]
    The proof is similar to the proof of Proposition \ref{prop:pred1}. We again
start with the uniqueness of $L'$. It is clear that the number of pixels is
either larger than $96$ or smaller or equal. Hence, $g$ is well defined by the
predicates $\beta_0$ and $\beta_1$. Next, we recognise that for images in
$\tilde{\mathcal M}$ the images with a vertical line sum up to $3 \cdot 32+32^2a$
and for the horizontal lines they sum up to $3\cdot 32 - 32^2a$. Hence, for all
all images $x \in\tilde{\mathcal M}$ we have
that those with horizontal lines have values $<96$ and for vertical lines $>96$
and therefore coincide with the evaluation of the structure $f$. Last, we see
that $\mathcal C$ in non-empty. All images in $x \in \mathcal{M}\setminus
\mathcal{\tilde{M}}$ with vertical lines have that the pixel sum is $96$.
Therefore, they also get classified as having a horizontal line, which is then
different to the classification by $f$.
\end{proof}

\subsection{Description of training procedures}
In this section we intend to describe the training procedure in detail, so that 
all the experiments becomes reproducible. A complete overview of the code, and 
the weights of the trained networks described in this paper can be downloaded from 
\url{https://github.com/vegarant/false_structures}. Before we start we would also 
like to point out that in each of the experiments there is an inherit randomness 
due to random initialization of the network weights, and ``on the fly''
generation of some of the training  and test sets. Hence, rerunning the code,
might result in slightly different results. It is beyond the scope of this
paper, to investigate how often, and under what circumstances a false structure
is learned. 

All the code have been implemented in Tensorflow version 1.13.1, and all layers
have used Tensorflows default weight initializer, the so called
``glorot\_uniform" initializer. If not otherwise stated, the default options 
for other parameters is always assumed.

\subsubsection{Experiment I}

\noindent {\bf Network architecture.} 
We considered two-layer neural networks with a ReLU activation function between
the two layers. The output dimension from the first layer was set to $4K$, and
the output dimension of the final layers was 1. An observant reader, who reads
the proof of Proposition \ref{prop:good_network}, might notice that it would be possible to
decrease the output dimension of the first layer. In our initial tests, we did 
this, but doing so made it substantially harder to learn the true structure $f_a$.  

\vspace{2mm}

\noindent {\bf Training parameters.}
We trained the networks using the cross entropy loss function for binary
classification as described in Proposition \ref{prop:good_network}. All networks was trained
using the ADAM optimizer, running  30000 epochs. The network trained on the set
containing only 7 samples used a batch size of 7, the three other networks used
a batch size of 50, with a shuffling of the data samples in each epoch.  

\vspace{2mm}

\noindent {\bf Training data.}
The training sets considered are $\mathcal{T}_{\delta}^{7},
\mathcal{T}_{0}^{5000}, \mathcal{T}_{\delta}^{5000}$ and
$\mathcal{T}_{\delta}^{10000}$, and is described in the main document. We point
out that for each of the training steps, each of these sets where drawn at
random. Hence there is some randomness in the experiment itself.  

\subsubsection{Experiment II}

\noindent {\bf Network architecture.} 
The trained network $\Phi$ had the following architecture.  
\begin{center}
\begin{tikzpicture}[scale = 1]
\tikzstyle{vertex} = [rectangle, fill=gray!30]
\tikzstyle{selected vertex} = [vertex, fill=red!50]
\tikzstyle{selected edge}= [draw,line width=1pt,-,red!100]
\tikzstyle{edge}= [->,black,line width=1pt]

\node[vertex] (v1) at (0,0) {Input};
\node[vertex] (v2) at (2,0) {Conv2D}; 
\node[vertex] (v3) at (4,0) {ReLU}; 
\node[vertex] (v4) at (6,0) {MaxPool}; 
\node[vertex] (v5) at (8,0) {Conv2D}; 
\node[vertex] (v6) at (10,0) {ReLU}; 

\node[vertex] (v7) at (2,-1) {MaxPool}; 
\node[vertex] (v8) at (4,-1) {Dense}; 
\node[vertex] (v9) at (6,-1) {Dense}; 

\draw[edge](v1)--(v2);
\draw[edge](v2)--(v3);
\draw[edge](v3)--(v4);
\draw[edge](v4)--(v5);
\draw[edge](v5)--(v6);
\draw[edge](v6)--(v7);
\draw[edge](v7)--(v8);
\draw[edge](v8)--(v9);
\end{tikzpicture}
\end{center}
where the to 2D convolutional layers (Conv2D) had a kernel size of 5, strides
equal $(1,1)$, padding equal ``same". The first of the convolutional layers had
24 filters, whereas the last had 48 filters. After each of the convolutional
layers we used a ReLU activation function. For the max pooling layers a pool
size of $(2,2)$ was used, with padding equal ``same". After the final max pool
layer the output of the layer was reshaped into a vector and feed to a dense
layer. The output dimension of the first dense layer was 10, whereas for the last 
it was 1.  

\vspace{2mm}

\noindent {\bf Training parameters.}
We trained the networks using the cross entropy loss function for binary
classification as described in Proposition \ref{prop:good_network}. The network
was trained using the ADAM optimizer for 10 epochs, with a batch size of 60. We
ran the code a few times to capture the exact false structure as described in
the main document, as it sometimes capture a slightly different false structure, which
does not have exactly 0\% success rate on $\mathcal{\widehat{C}}^{b,c}$,
$b=0.009$ and $c=0.01$.  
 
\vspace{2mm}

\noindent {\bf Training data.}
The network was trained on the training set $\mathcal{T} \subset
\mathcal{\tilde{M}}$ which contains exactly the 60 unique elements in
$\mathcal{\tilde{M}}$ for which $a = 0.01$.

\end{document}